\documentclass[letterpaper, 10pt, conference]{ieeeconf}  

\IEEEoverridecommandlockouts                               
\let\proof\relax

\let\labelindent\relax

\usepackage{amsmath}
\usepackage{amssymb}
\usepackage{amsthm}
\usepackage{enumitem}
\usepackage[dvipsnames]{xcolor}
\usepackage{float}
\newcommand{\nonl}{\renewcommand{\nl}{\let\nl\oldnl}}
\usepackage{soul}
\usepackage{graphicx}
\usepackage{booktabs}
\usepackage{balance}
\usepackage{comment}
\usepackage{kz_style}

\newcommand{\ra}{\rightarrow}
\newcommand{\mb}[1]{\mathbb{#1}}

\newcommand{\mc}[1]{\mathcal{#1}}

\newcommand{\pp}{\mathbb{P}}

\title{\LARGE \bf Information State Embedding in Partially Observable Cooperative Multi-Agent Reinforcement Learning\thanks{Research supported in part by the US Army Research Laboratory (ARL) Cooperative Agreement W911NF-17-2-0196, and in part by the Office of Naval Research (ONR) MURI Grant N00014-16-1-2710. Authors are with the Coordinated Science Laboratory, University of Illinois, Urbana, IL 61801, U.S.A.  Email addresses: \{weichao2, kzhang66, miehling, basar1\}@illinois.edu.}}

\author{Weichao Mao, Kaiqing Zhang, Erik Miehling, Tamer Ba\c{s}ar}

\begin{document}

	\maketitle
	\thispagestyle{empty}
	\pagestyle{empty}
	
	\begin{abstract} 
	Multi-agent reinforcement learning (MARL) under partial observability has long been considered challenging, primarily due to the requirement for each agent to maintain a belief over all other agents' local histories -- a domain that generally grows exponentially over time. In this work, we investigate a partially observable MARL problem in which agents are cooperative. To enable the development of tractable algorithms, we introduce the concept of an \emph{information state embedding} that serves to compress agents' histories. We quantify how the compression error influences the resulting value functions for decentralized control. Furthermore, we propose an instance of the embedding based on recurrent neural networks (RNNs). The embedding is then used as an approximate information state, and can be fed into any MARL algorithm. The proposed \emph{embed-then-learn} pipeline opens the black-box of existing (partially observable) MARL algorithms, allowing us to establish some theoretical guarantees (error bounds of value functions) while still achieving competitive performance with many end-to-end approaches.

	\end{abstract}

\section{Introduction}\label{sec:intro}
Multi-agent reinforcement learning (MARL) is a prominent and widely applicable paradigm for modeling multi-agent sequential decision making under uncertainty, with applications in a wide range of domains including robotics~\cite{duan2012multi}, cyber-physical systems~\cite{wang2016towards}, and finance~\cite{lee2007multiagent}. Many practical problems require agents to make decisions based on only a partial view of the environment and the (private) information of other agents, e.g., intention of pedestrians in autonomous driving settings, or location of surveillance targets in military drone swarm applications. This partial information generally precludes agents from making optimal decisions due to the requirement that each agent maintains a belief over all other agents' local histories -- a domain that, in general, grows exponentially in time \cite{nayyar2014common}.

We consider a partially observable setting, but restrict attention to problems in which the agents are cooperative, that is, they share the same objective.\footnote{This is in contrast with the more challenging general case of non-cooperative agents in which agents may act strategically to achieve their (individual) goals.} Even under this simpler (cooperative) setting, the primary challenge still exists: due to the lack of explicit communication, agents possess noisy and asymmetric information about the environment, yet their rewards depend on the joint actions of all agents. As in the general (non-cooperative) setting, this informational coupling requires agents to maintain a growing amount of information.

Many approaches have been proposed to address this challenge; we roughly categorize them into two classes: \textit{concurrent learning} and \textit{centralized learning}. In concurrent learning approaches, agents learn and update their own control policies simultaneously. However, since reward and state processes are coupled, the environment becomes non-stationary from the perspective of each agent, and hence concurrent solutions do not converge in general~\cite{gupta2017cooperative}. On the other hand, centralized learning approaches, as the name suggests, reformulate the problem from the perspective of a centralized (or virtual) coordinator~\cite{nayyar2013decentralized}. 
Despite its popularity~\cite{dibangoye2016optimally,dibangoye2018learning}, the centralized approach suffers from high computational complexity. A centralized algorithm needs to assign an action to each possible history sequence of the agents, and the cardinality of such sequences grows exponentially over time. In fact, decentralized partially observable Markov decision processes (Dec-POMDPs), an instance of the general decentralized control model, are known to be NEXP-complete~\cite{bernstein2002complexity}.

In this paper, we propose to address the computational issues in the centralized approach by extracting a summary, termed an \textit{information state embedding}, from the history space, and then learning control policies in the compressed embedding space that possess some quantifiable performance guarantee. This procedure, which we term the \emph{embed-then-learn} pipeline, is depicted in Fig.~\ref{fig:pipeline}. In the first stage, an embedding scheme with a quantifiable compression error is extracted from the history space, 
where our metric of compression error, to be defined in Section~\ref{sec:embedding}, favors an embedding with higher predictive ability. 
In the second stage, we learn a policy in this compressed embedding space. We prove how the embedding error propagates over time, and our theoretical results provide an overall performance bound of the policy in terms of the compression error. Therefore, in this paradigm, the cooperative MARL problem can be reduced to one of finding an information state embedding with a small compression error. 

\begin{figure*}[!t]
	\centering\includegraphics[width=.98\linewidth]{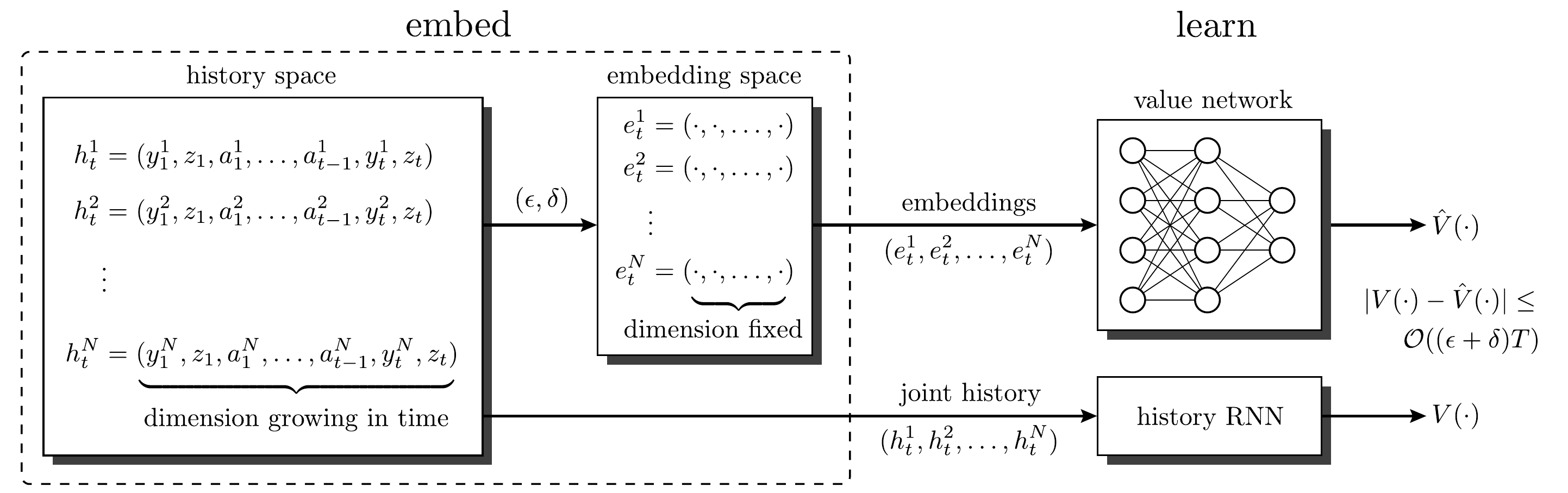}
	\caption{The \emph{embed-then-learn} pipeline. First, the \emph{embed} stage extracts a compact summary, given $\epsilon$ and $\delta$ tolerances (as discussed in Definition \ref{def:embedding}), from the history space, followed by the \emph{learn} stage in which embedding-based policies are learned. Note that the lower path (the history-based policies) is not actually carried out; it is only illustrated to describe the bound.}\label{fig:pipeline}
\end{figure*}

We also introduce an empirical instance of the information state embedding, and demonstrate how to extract an embedding from data. 
Despite the empirical results that \emph{end-to-end} learning\footnote{In deep learning, end-to-end learning generally means training one single neural network that takes in the raw data as input and outputs the ultimate goal task.} generally leads to state-of-the-art performance~\cite{hausknecht2015deep}, our approach breaks a partially observable MARL problem into two stages: embedding followed by learning. It provides a new angle that also enjoys fair performance, while allowing for a more theoretical foundation. This  approach opens the  \emph{black-box} in end-to-end approaches, and  makes an initial step toward understanding their great empirical successes. 

\vspace{5pt}
\noindent\textbf{Related Work.}
In the seminal work~\cite{witsenhausen1973standard}, it was proved that a decentralized problem can be converted to a centralized standard form and hence be solved via a dynamic program. Following this paradigm, the common information approach~\cite{nayyar2013decentralized} formulates the decentralized control problem as a single agent POMDP based on the common information of all the agents. A similar work~\cite{dibangoye2016optimally} transforms the Dec-POMDP into a continuous-state MDP, and again demonstrates that standard techniques from POMDPs are applicable. More recently, \cite{tavafoghi2018unified} has introduced a sufficient information approach that investigates the sufficient conditions under which an information state is optimal for decision making. However, this approach does not offer a constructive algorithm of the sufficient information state. In fact, it is generally difficult, in multi-agent settings, to determine whether an information state that is more compact than the entire history exists. Our work extends this approach, in that we \emph{learn} or \emph{extract} a compact information state \emph{from data} by approximately satisfying these sufficient conditions for optimal decision-making. Perhaps most related to our work is~\cite{dibangoye2014error}, where the authors also studied error propagation for approximations in Dec-POMDPs. Our work differs and complements~\cite{dibangoye2014error} in that~\cite{dibangoye2014error} approximates the Bellman backup operator and the transition matrix, while we focus on the approximation of the information state. 

Learning a state representation that is more compact than an explicit history is also of great interest even in single agent partially observable systems. When the underlying system has an inherent state, as in POMDPs, it is helpful to directly learn a generative model of the system~\cite{ma2017pflstm,moreno2018neural}. When an inherent system state is not available, it is still possible to learn an internal state representation that captures the system dynamics. Two well-known approaches are predictive state representation (PSR)~\cite{littman2002predictive} and causal state representation learning~\cite{zhang2019learning}. Interested readers are referred to~\cite{lesort2018state} for a detailed survey of state representation learning in control. 
A recent work~\cite{subramanian2019approximate} points out that PSR is not sufficient for RL problems, and proposes to extend PSR by introducing a set of sufficient conditions for performance evaluation. 
In comparison, we generalize their analysis to the multi-agent setting, where our work can be regarded as learning a state representation in a \emph{decentralized} partially observable system. 

There is also no lack of empirical/heuristic approaches for MARL under partial observability. In concurrent learning, a learning scheme is proposed in~\cite{banerjee2012sample} where agents take turns to learn the optimal response to others' policies. The authors in~\cite{gupta2017cooperative} extend single-agent RL algorithms to the multi-agent setting and empirically evaluate their performance. Another work~\cite{omidshafiei2017deep} studies multi-task MARL and introduces 
a multi-agent extension of experience replay~\cite{mnih2015human}. A centralized learning algorithm, termed oSARSA, is proposed in~\cite{dibangoye2018learning} to solve a centralized continuous-state MDP formulated from the MARL problem. For a more detailed survey, please refer to~\cite{zhang2019multi}. Compared with those more end-to-end solutions, our embed-then-learn pipeline enables more theoretical analysis. 

\vspace{5pt}
\noindent\textbf{Contribution.} We summarize our contributions as follows: {\bf 1)} We address the computational issue in partially observable MARL, by compressing each agent's local history to an information state embedding that lies in a fixed domain. {\bf 2)} Given the compression error of an embedding, we prove that the value function defined on the embedding space has bounded approximation error. {\bf 3)} We propose to learn an embedding instance using RNNs, and empirically evaluate its performance on some benchmark Dec-POMDP tasks.


\vspace{5pt}
\noindent\textbf{Notation.} We use uppercase letters (e.g., $X$) to denote random variables, lowercase letters (e.g., $x$) for their corresponding realizations, and calligraphic letters (e.g., $\mc{X}$) for the corresponding spaces where the random variables reside. Subscripts denote time indices, whereas superscripts index agents. For $a< b$, the notation $a:b$ denotes $\{a,a+1,\dots, b\}$, e.g., $X_{a:b}$ is a shorthand for $\{X_a, X_{a+1},\dots, X_b\}$, with a similar convention for superscripts. Finally, $\Delta(\mc{X})$ denotes the set of all probability measures over $\mc{X}$.

\section{Model and Preliminaries}\label{sec:preliminaries}
We adopt a model similar to the one developed in \cite{tavafoghi2018unified}, and for completeness we state it here. Consider a team of $N$ non-strategic, i.e., cooperative, agents indexed by $i\in \mathcal{I}=\{1,2,\dots,N\}$ 
over a finite horizon $\mc{T} = \{1, \dots T\}$. 
The state of the environment $X_t\in\mc{X}_t$, the private observation $Y_t^i\in \mc{Y}_t^i$ of agent $i$, and the common observation\footnote{This definition of common observations can be equivalently considered as the innovations defined in the partial history sharing setting~\cite{zhang2019online,nayyar2013decentralized}, where each agent sends a subset of its local history to a shared memory according to a predefined sharing protocol.} of all agents $Z_t\in\mc{Z}_t$, follow the following dynamics:
\[
\begin{aligned}
X_{t+1}&=f_{t}\left(X_{t}, A_{t}^{1:N}, W_{t}^{x}\right),\\
Y_{t+1}^{i}&=l_{t+1}^{i}\left(X_{t+1}, A_{t}^{1:N}, W_{t+1}^{i}\right),\\
Z_{t+1}&=l_{t+1}^{c}\left(X_{t+1}, A_{t}^{1:N}, W_{t+1}^{c}\right),
\end{aligned}
\]
where $X_1, W_t^x, W_t^i$ and $W_t^c$ are independent random variables for all $t\in\mc{T},i\in\mc{I}$. We also assume $\mc{X}_t, \mc{Y}_t^i, \mc{A}_t^i, \mc{Z}_t$ are all finite sets. The initial state $X_1$ follows a fixed distribution $f_X\in\Delta(\mc{X}_1)$. Note that the Dec-POMDP model {of}~\cite{oliehoek2016concise} considers the special case where the agents share no common information, i.e., $\mc{Z}_t$ is empty.

Define common information $C_t = \{Z_{1:t}\} \in \mc{C}_t$ as the aggregate of common observations from time $1$ to $t$. Similarly, let $P_t^i = \{Y_{1:t}^i, A_{1:t-1}^i\} \backslash C_t \in \mc{P}_t^i$ denote each agent's \textit{private information}, assumed to be unknown to the other agents $j\neq i$, where $\mc{P}_t^i$ is the space of agent $i$'s private information at time $t$. Define the joint history $H_t = \{Y_{1:t}^{1:N}, A_{1:t-1}^{1:N}, Z_{1:t}\}\in \mc{H}_t$ to be the collection of all agents' actions and observations up to and including time $t$. Accordingly, define each agent's \textit{local history} $H_{t}^{i}=\left\{Y_{1:t}^i, A_{1:t-1}^{i},Z_{1:t}\right\} \in \mathcal{H}_{t}^{i}$ to be agent $i$'s information at time $t$. 
Under the assumption of perfect recall, each agent's strategy $g_{1:T}^i$ maps its local history to a distribution over its actions, i.e., $g_t^i:\mc{H}_t^i \rightarrow \Delta(\mc{A}_t^i)$. We also refer to $g_t^i$ as agent $i$'s control policy at $t$ and $g_{1:T}^i$ as agent $i$'s control law.


\begin{remark}
	We consider stochastic policies rather than deterministic policies throughout the paper. To understand the reason, consider the single agent case, a standard POMDP. Even though deterministic policies are known to be optimal in POMDPs, their existence relies on knowing the belief state with certainty. In any learning approach, the belief state may not be known with certainty at any time. Any approximation to the true belief state may give rise to stochastic policies that outperform deterministic policies (an extreme case is when the action is based on only the most recent observation as in~\cite{singh1994learning}). Since a Dec-POMDP can be translated to an equivalent (centralized/single-agent) POMDP \cite{nayyar2013decentralized}, the requirement to consider stochastic policies when only an approximate belief state is available carries over to the multi-agent setting as well.
\end{remark}

At each time $t$, all the agents receive a joint reward $R_t(X_t, A_t^{1:N})$. The agents' joint objective is to maximize the total reward over the entire horizon:
$
R(g) = \mb{E}_g\left[\sum_{t=1}^{T}R_t(X_t, A_t^{1:N})\right], 
$
where the expectation is taken with respect to the probability distribution on the states induced by the joint control law $g\triangleq g_{1:T}^{1:N}$. 

It has been shown in~\cite{nayyar2013decentralized} that this decentralized control problem can be formulated as a centralized POMDP from the perspective of a virtual coordinator. The coordinator only has access to the common information $C_t$, not the agents' private information $P_t^{1:N}$. The centralized state $(X_t, P_t^{1:N})\in \mc{X}_t \times \mc{P}_t^{1:N}$ corresponds to the environment state and agents' private information in the decentralized problem. The centralized actions, termed prescriptions (to be defined in Definition~\ref{def:prescriptions}), are mappings from each agent's private information to a local control action, i.e., $\gamma_{t}^i:\mc{P}_t^i \ra \Delta(\mc{A}_t^i)$. The information state in the centralized POMDP is a belief over the environment state and all the agents' private information, conditional on the common information and joint strategies (see e.g., Lemma 1 in~\cite{nayyar2013decentralized}). For optimal decision making, the coordinator needs to maintain a belief over all the agents' private information, the domain of which grows exponentially over time. Therefore, the centralized approach is generally intractable. In the following, we review relevant definitions and structural results from the literature. 
\begin{definition}[Sufficient private and common information~\cite{tavafoghi2018unified}]\label{def:sufficient}
	We say $S_{t}^{i}$, which is a function of $P_{t}^{i}$ and $C_{t}$, denoted by $\zeta_{t}^{i}\left(P_{t}^{i}, C_{t} ; g\right), i \in \mathcal{N}, t \in \mathcal{T}$, is sufficient private information if it satisfies the following conditions:
	\begin{enumerate}[label=(\alph*)]
		\item Recursive update:
		$$S_{t}^{i}=\phi_{t}^{i}\left(S_{t-1}^{i}, H_{t}^{i} \backslash H_{t-1}^{i} ; g\right), \forall t \in \mathcal{T} \backslash\{1\},$$
		\item Sufficient to predict future observations:
		\ifx\singlecolumn\undefined 
		\[
		\begin{aligned} 
		&\mathbb{P}^{g}\{S_{t+1}^{1:N} = s_{t+1}^{1:N}, Z_{t+1} = z_{t+1} | P_{t}^{1:N}, C_{t}, A_{t}^{1:N}\}\\
		&\qquad =\mathbb{P}^{g}\{S_{t+1}^{1:N} = s_{t+1}^{1:N}, Z_{t+1} = z_{t+1} | S_{t}^{1:N}, C_{t}, A_{t}^{1:N}\},
		\end{aligned} 
		\]
		\else 
		\[
		\begin{aligned} 
		\mathbb{P}^{g}\{S_{t+1}^{1:N} = s_{t+1}^{1:N}, Z_{t+1} = z_{t+1} | P_{t}^{1:N}, C_{t}, A_{t}^{1:N}\}=\mathbb{P}^{g}\{S_{t+1}^{1:N} = s_{t+1}^{1:N}, Z_{t+1} = z_{t+1} | S_{t}^{1:N}, C_{t}, A_{t}^{1:N}\},
		\end{aligned} 
		\]
		\fi  
		
		$\forall \{s_{t+1}^{1:N}, z_{t+1}\} \in \mathcal{S}_{t+1}^{1:N} \times \mathcal{Z}_{t+1}$ and $\forall g\in\mc{G}$, 
		\item Sufficient to predict future reward: For any $\tilde{g}$ (strategies defined on the sufficient information space),
		\[
		\mathbb{E}^{\tilde{g}}\left[R_t \mid C_{t}, P_{t}^{i}, A_{t}^{1:N}\right]=\mathbb{E}^{\tilde{g}}\left[R_t \mid C_{t}, S_{t}^{i}, A_{t}^{1:N}\right],
		\]
		\item Sufficient to predict others' private information:
		\[
		\mathbb{P}^{\tilde{g}}\left\{S_t^{-i} = s_{t}^{-i} | P_{t}^{i}, C_{t}\right\}=\mathbb{P}^{\tilde{g}}\left\{S_t^{-i} = s_{t}^{-i} | S_{t}^{i}, C_{t}\right\},
		\]
		$\forall s_t^{-i} \in \mathcal{S}_{t}^{-i}$ and $\forall \tilde{g}\in\tilde{\mc{G}}.$
	\end{enumerate} 
Sufficient common information $\Pi_t$ is defined as the conditional distribution over the environment state $X_t$ and the joint private information $S_t^{1:N}$ of all the agents, given the current common information $C_t$ and previous strategies $g_{1:t-1}$:
	\[
	\Pi_t=\mb{P}^{g_{1: t-1}}\left(X_{t}, S_{t}^{1:N} \mid C_{t}\right).
	\] 
\end{definition} 
Due to the above definition, an agent can make decisions using only its sufficient private information and common information. Such a decision making strategy $\sigma_t^i:\Delta(\mc{X}_t\times \mc{S}_t^{1:N})\times \mc{S}_t^i \rightarrow \Delta(\mc{A}_t^i)$ is termed a sufficient information based (SIB) strategy for agent $i$ at time $t$. 


Let $\eta^{g_{1:t-1}}:\mc{C}_t \rightarrow \Delta(\mc{X}_t \times \mc{S}_t^{1:N})$ denote the mapping from common information to sufficient common information, i.e., $\Pi_t = \eta^{g_{1:t-1}}(C_t)$. The authors of \cite{tavafoghi2018unified} showed that the SIB belief $\Pi_t$ can be updated recursively via Bayes' rule, using only the previous common belief $\Pi_{t-1}$ and the new common observation $Z_t$. That is, there exists a mapping $\psi_t^{\sigma_{t-1}}:\Delta(\mc{X}_{t-1}\times \mc{S}_{t-1}^{1:N}) \times \mc{Z}_t \ra \Delta(\mc{X}_t \times \mc{S}_t^{1:N})$, such that $\Pi_{t}=\psi_{t}^{\sigma_{t-1}}\left(\Pi_{t-1}, Z_{t}\right)$. 
When the belief $\Pi_t$ is fixed, the SIB strategy  $\sigma_t^i(\Pi_t, S_t^i)$ only depends on $S_t^i$. This induces another function, termed prescription. 
\begin{definition}[Prescriptions~\cite{nayyar2013decentralized}]\label{def:prescriptions}
	A \textit{prescription} $\gamma_t^i: \mc{S}_t^i \rightarrow \Delta(\mc{A}_t^i)$ is a mapping from agent $i$'s sufficient private information to a distribution over actions at time $t$. 
\end{definition}

According to Theorem 3 in~\cite{tavafoghi2018unified}, given perfect information of the system dynamics (i.e., $f_t,l_{t}^i$, and $l_{t}^c$), the optimal planning  solution to the decentralized control problem can be found via the following dynamic program:
\begin{equation}
V_{T+1}\left(\pi_{T+1}\right)=0, \quad \forall \pi_{T+1}\in\mathit{\Pi}_{T+1},
\end{equation}
and at every $t\in \mc{T}$,
\ifx\singlecolumn\undefined 
\begin{align}
V_{t}\left(\pi_{t}\right) &=\underset{\gamma_t^{1: N}: \mc{S}_{t}^{1: N}\ra \Delta(\mc{A}_t^{1:N})}{\max } Q_t(\pi_t, \gamma_t^{1:N}),\forall \pi_t \in\mathit{\Pi}_t,\nonumber\\
Q_t(\pi_{t}, \gamma_t^{1:N} ) &=\mathbb{E}  \Big[R_{t}(X_{t}, \gamma_t^{1: N}(S_{t}^{1: N})) \nonumber\\
& \hspace{-0.6cm}+V_{t+1}(\psi_{t+1}^{\gamma_{t}}(\pi_{t}, Z_{t+1}))\Big| \Pi_t = \pi_t \Big],\forall \pi_t \in\mathit{\Pi}_t.\label{eqn:dp}
\end{align}
\else 
\begin{align}
V_{t}\left(\pi_{t}\right) &=\underset{\gamma_t^{1: N}: \mc{S}_{t}^{1: N}\ra \Delta(\mc{A}_t^{1:N})}{\max } Q_t(\pi_t, \gamma_t^{1:N}),\forall \pi_t \in\mathit{\Pi}_t,\nonumber\\
Q_t(\pi_{t}, \gamma_t^{1:N} ) &=\mathbb{E}  \Big[R_{t}(X_{t}, \gamma_t^{1: N}(S_{t}^{1: N})) +V_{t+1}(\psi_{t}^{\gamma_{t}}(\pi_{t}, Z_{t+1}))\Big| \Pi_t = \pi_t \Big],\forall \pi_t \in\mathit{\Pi}_t.\label{eqn:dp}
\end{align}
\fi  
In a learning problem, the system model is unknown to the agents. Agents must infer this information through interaction with the environment. The remainder of the paper is devoted to solving the learning problem.

\section{Information State Embedding}\label{sec:embedding}

The definition of sufficient information (see Definition~\ref{def:sufficient}) characterizes a compression of history that is sufficient for optimal decision making. However, it does not offer an explicit way to construct such an information state, nor to learn it from data. In this section, we first define the notion of an \textit{information state embedding}, an approximate version of sufficient information, and analyze the approximation error it introduces. We further provide an explicit algorithm to learn this information state embedding from data. 

\subsection{Information State Embedding: Definition}
Since a Dec-POMDP can be formulated as a centralized POMDP~\cite{nayyar2013decentralized}, it might seem reasonable to apply (single-agent) state representation techniques, e.g., ~\cite{subramanian2019approximate}, directly to the centralized problem to learn an appropriate information state embedding. However, this is not very helpful because the ``state'' in the centralized POMDP is derived from the common information in the decentralized problem, and hence single-agent state representation techniques would only compress common information. However, the intractability of the decentralized problem arises from the exponential growth of private information. To address the computational bottleneck in the decentralized problems, what we really need is a compact embedding of the agents' private information. 

Let $\hat{S}_t^i$ denote a compression of agent $i$'s sufficient private information $S_t^i$ at time $t$, where the detailed properties of this compression will be clear later.
This compression mapping is denoted by $\alpha_t^i: \mc{S}_t^i\ra \hat{\mc{S}}_t^i$, and we assume that $\alpha_t^i$ is injective\footnote{A function $f:X\hspace{-1mm}\ra\hspace{-1mm} Y$ is injective if $\forall a,b\in X, f(a)=f(b)\Rightarrow a=b.$} for all $i \in \mc{I}, t\in \mc{T}$. It is worth noting that the injectivity assumption is also inherent in the mathematical definition of embeddings in topological spaces (see e.g.,~\cite{sankappanavar1981course}). 
\begin{remark}
	We note that the injectivity assumption does not contradict the fact that $\alpha_t^i$ is a compression. The injective property restricts the cardinalities of the domain and the co-domain, yet compression concerns dimensionality. For computational reasons, we assume throughout the paper that $\hat{\mc{S}}^i$ has a fixed domain.\footnote{A fixed domain is not a theoretical requirement in our analysis, it is just desirable from a computational perspective.} Given its private information embedding, an agent makes decisions using its \textit{embedded strategy} defined as $\hat{g}_t^i: \hat{\mc{S}}_t^i \times \mc{C}_t \ra \Delta(\mc{A}_t^i)$.
\end{remark}

Define the compressed common belief $\hat{\Pi}_t$ to be the conditional distribution over the current environment state $X_t$ and the joint private information embeddings $\hat{S}_t^{1:N}$ of all the agents, given the current common information $C_t$ and previous embedded strategies $\hat{g}_{1:t-1}$, i.e., $\hat{\Pi}_t=\mb{P}^{\hat{g}_{1: t-1}}\left(X_{t}, \hat{S}_{t}^{1:N} \mid C_{t}\right)$. Define the embedded Bayes' rule $\hat{\psi}_t^{\hat{\gamma}_{t-1}}:\Delta(\mc{X}_{t-1}\times \hat{\mc{S}}_{t-1}) \times \mc{Z}_t \ra \Delta(\mc{X}_t \times \hat{\mc{S}}_t)$ analogously. 
Following Definition~\ref{def:prescriptions}, we define the common information compression mapping $\hat{\eta}^{\hat{g}_{1:t-1}}:\mc{C}_t \rightarrow \Delta(\mc{X}_t \times \hat{\mc{S}}_t^{1:N})$ and embedded prescriptions $\hat{\gamma}_t^i: \hat{\mc{S}}_t^i \rightarrow \Delta(\mc{A}_t^i)$ accordingly. 
Analogous to~\eqref{eqn:dp}, we can also define a dynamic program based on our embedded information:
\begin{equation}
\hat{V}_{T+1}\left(\hat{\pi}_{T+1}\right)=0, \quad \forall \hat{\pi}_{T+1}\in\hat{\mathit{\Pi}}_{T+1},
\end{equation}
and for every $\hat{\pi}_t$ at every $t\in \mc{T}$,
\ifx\singlecolumn\undefined 
\begin{equation} \label{eqn:dp_embedded}
\begin{aligned}
\hat{V}_{t}\left(\hat{\pi}_{t}\right) &=\underset{\hat{\gamma}_t^{1: N}
}{\max\ } \hat{Q}_t(\hat{\pi}_t, \hat{\gamma}_t^{1:N}),\\
\hat{Q}_t(\hat{\pi}_{t}, \hat{\gamma}_t^{1:N} ) &=\mathbb{E}  \Big[R_{t}(X_{t}, \hat{\gamma}_t^{1: N}(\hat{S}_{t}^{1: N}))\\
&\qquad +\hat{V}_{t+1}(\hat{\psi}_{t+1}^{\hat{\gamma}_{t}}(\hat{\pi}_{t}, Z_{t+1}))\mid \hat{\Pi}_t = \hat{\pi}_t \Big].
\end{aligned}
\end{equation}
\else 
\begin{equation} \label{eqn:dp_embedded}
\begin{aligned}
\hat{V}_{t}\left(\hat{\pi}_{t}\right) &=\underset{\hat{\gamma}_t^{1: N}: \hat{\mc{S}}_{t}^{1: N}\ra \Delta(\mc{A}_t^{1:N})
}{\max\ } \hat{Q}_t(\hat{\pi}_t, \hat{\gamma}_t^{1:N}),\forall \hat{\pi}_t \in\mathit{\hat{\Pi}}_t\\
\hat{Q}_t(\hat{\pi}_{t}, \hat{\gamma}_t^{1:N} ) &=\mathbb{E}  \Big[R_{t}(X_{t}, \hat{\gamma}_t^{1: N}(\hat{S}_{t}^{1: N}))+\hat{V}_{t+1}(\hat{\psi}_{t}^{\hat{\gamma}_{t}}(\hat{\pi}_{t}, Z_{t+1}))\mid \hat{\Pi}_t = \hat{\pi}_t \Big],\forall \hat{\pi}_t \in\mathit{\hat{\Pi}}_t.
\end{aligned}
\end{equation}
\fi  
To quantify the performance of any such information-embedding, we formally define an  $(\epsilon, \delta)$-information state embedding as follows.

\begin{definition}[$(\epsilon, \delta)$-information state embedding]\label{def:embedding}
	We call $\{\hat{S}^{1:N}_t, \hat{\Pi}_t\}$ an $(\epsilon,\delta)$-\emph{information state embedding} if it satisfies the following two conditions:
	\begin{enumerate}[label=(\alph*)]
		\item Approximately sufficient to predict future rewards:
		For any $t\in \mc{T}$ and any realization of sufficient private information $s_t^{1:N}$, common information $c_t$, and actions $a_t^{1:N}$:
		\ifx\singlecolumn\undefined 
		\[
		\begin{aligned} 
		\big| \mb{E}[ R_t(X_t, &a_t^{1:N}) \mid  \pi_t, s_t^{1:N}, a_t^{1:N}]\\
		&-\mb{E}[ R_t(X_t, a_t^{1:N}) \mid \hat{\pi}_t, \hat{s}_t^{1:N}, a_t^{1:N}] \big|\leq \epsilon.
		\end{aligned} 
		\]
		\else 
		\[
		\begin{aligned} 
		\big| \mb{E}[ R_t(X_t, a_t^{1:N}) \mid  \pi_t, s_t^{1:N}, a_t^{1:N}]-\mb{E}[ R_t(X_t, a_t^{1:N}) \mid \hat{\pi}_t, \hat{s}_t^{1:N}, a_t^{1:N}] \big|\leq \epsilon.
		\end{aligned} 
		\]
		\fi  
		\item Approximately sufficient to predict future beliefs: 
		For any Borel subset $B \subseteq \Delta(\mc{X}_{t+1}\times \hat{\mc{S}}_{t+1})$, define	
		\[
		\begin{aligned} 
		\mu_t(B;a_t^{1:N}) &= \mb{P}\left( \hat{\Pi}_{t+1} \in B \mid  \pi_t,s_t^{1:N}, a_t^{1:N} \right),\\
		\nu_t(B;a_t^{1:N}) &= \mb{P}\left( \hat{\Pi}_{t+1} \in B \mid \hat{\pi}_t, \hat{s}_t^{1:N}, a_t^{1:N}  \right). 
		\end{aligned} 
		\]
		Then 
		\[
		\mc{K}\left(\mu_t(a_t^{1:N}),\nu_t(a_t^{1:N})\right) \leq \delta,
		\]
		where $\mc{K}$ denotes the Wasserstein or Kantorovich-Rubinstein distance between two distributions. 
	\end{enumerate}
\end{definition}

By Kantorovich-Rubinstein duality~\cite{edwards2011kantorovich}, Definition~\ref{def:embedding} suggests $\left|\int f d \mu_t-\int f d \nu_t\right| \leq \delta$ for any Lipschitz continuous function $f$ with Lipschitz constant $\|f\|_{Lip} \leq 1$ (with respect to the Euclidean metric). To obtain an error bound on the value function, we make the following assumption: 

\begin{assumption}[Lipschitz continuity of value functions]\label{asm:lipschitz}
	Value functions $\hat{V}_t:\Delta(\mc{X}_{t}\times \hat{\mc{S}}_{t}^{1:N}) \ra \mb{R}$ are Lipschitz continuous for all $t \in \mc{T}$ with Lipschitz constant upper bound $L_V$, i.e., $\|\hat{V}(\hat{\pi}_t) - \hat{V}(\hat{\pi}_t^\prime)\|_2 \leq L_V \|\hat{\pi}_t - \hat{\pi}_t^\prime \|_2$. 
\end{assumption}

We note that Lipschitz continuity over the compressed common information space is a mild assumption. This is because by centralizing the problem as a single agent POMDP, the value function is piecewise linear convex over the belief state~\cite{sondik1971optimal}, which is Lipschitz continuous over the (non-compressed) common information space. 

\begin{corollary}\label{remark}
Combining Definition~\ref{def:embedding}(b) and Assumption~\ref{asm:lipschitz}, we know that for any realization $c_t,s_t^{1:N}$ and $a_t^{1:N}$, $\forall t\in\mc{T}$:
\ifx\singlecolumn\undefined 
\[
\begin{aligned} 
\Big|\mb{E}[ \hat{V}(\hat{\Pi}_{t+1})\mid & \pi_t ,s_t^{1:N} , a_t^{1:N} ]\\
& -\mb{E}[ \hat{V}(\hat{\Pi}_{t+1})\mid \hat{\pi}_t ,\hat{s}_t^{1:N} , a_t^{1:N} ] \Big| \leq L_V\delta.
\end{aligned} 
\]
\else 
\[
\begin{aligned} 
\Big|\mb{E}[ \hat{V}(\hat{\Pi}_{t+1})\mid \pi_t ,s_t^{1:N} , a_t^{1:N} ] -\mb{E}[ \hat{V}(\hat{\Pi}_{t+1})\mid \hat{\pi}_t ,\hat{s}_t^{1:N} , a_t^{1:N} ] \Big| \leq L_V\delta.
\end{aligned} 
\]
\fi  
\end{corollary}

\subsection{Information State Embedding: Error Analysis}

Next, we extend the approximation error analysis in~\cite{subramanian2019approximate} to the multi-agent setting. Our main result is that, by compressing the exponentially growing history to an $(\epsilon, \delta)$-information state embedding, the error of value functions over the entire horizon is bounded as stated in the theorem below:
\begin{theorem}\label{thm:dpbound}
	For any $t \in \mc{T}$ and any realization $c_t$ and $p_t^{1:N}$, let $\gamma_{t}^{_*1:N}$ and $\hat{\gamma}_{t}^{_*1:N}$ denote the optimal prescriptions in the two dynamic programming solutions~\eqref{eqn:dp} and \eqref{eqn:dp_embedded}, respectively. Then, we have:
	\[
	\begin{aligned} 
	\left| Q_t(\pi_t, \gamma_t^{_*1:N}) - \hat{Q}_t(\hat{\pi}_t, \hat{\gamma}_t^{_*1:N}) \right| &\leq (T-t+1)(\epsilon+L_V\delta),\\
	\left| V_t(\pi_t) - \hat{V}_t(\hat{\pi}_t) \right| &\leq (T-t+1)(\epsilon+L_V\delta).\\
	\end{aligned} 
	\]
\end{theorem}
\begin{proof}
	We prove by backward induction. As the basis of induction, Theorem~\ref{thm:dpbound} holds at time $T+1$ by construction. Suppose Theorem~\ref{thm:dpbound} holds at time $t+1$, $t\leq T$; then, for time $t$, we define an auxiliary set of prescriptions $\hat{\gamma}_{t}^{_01:N}$ that produces exactly the same action distribution as $\gamma_{t}^{_*1:N}$. Specifically, for any $t\in\mc{T}, i\in\mc{I}$, and for any realization of $S_t^i$, this definition implies $\pp \left(a_t^{i}\mid \hat{s}_t^{i},\hat{\gamma}_t^{_0i}\right) = \pp \left(a_t^{i}\mid s_t^{i},\gamma_t^{_*i}\right),\forall a_t^{i}\in \mc{A}_t^{i}$. This is possible because our private information embedding process is assumed to be injective. The existence of such an oracle $\hat{\gamma}_{t}^{_01:N}$ suggests that, given only the embedded information, it is always possible to recover the optimal action distributions that are produced by the complete information. Nevertheless, our embedded-information-based dynamic program ends up generating a different set of prescriptions $\hat{\gamma}_{t}^{_*1:N}$, and hence the oracle $\hat{\gamma}_{t}^{_01:N}$ is only used for analysis purposes. Together with Definition~\ref{def:embedding}(a), the following holds for every realization $c_t$ and $s_t^{1:N}$:
	\ifx\singlecolumn\undefined 
	\begin{equation}\label{eqn:reward}
	\begin{aligned} 
	\big| \mb{E}[ R_t(&X_t, \gamma_t^{_*1:N}(s_t^{1:N})) \mid  \pi_t, s_t^{1:N}, \gamma_t^{_*1:N}] -\\
	&\mb{E}[ R_t(X_t, \hat{\gamma}_t^{_01:N}(\hat{s}_t^{1:N})) \mid \hat{\pi}_t, \hat{s}_t^{1:N}, \hat{\gamma}_t^{_01:N}] \big|\leq \epsilon,
	\end{aligned} 
	\end{equation}
	\else 
	\begin{equation}\label{eqn:reward}
	\begin{aligned} 
	\big| \mb{E}[ R_t(X_t, \gamma_t^{_*1:N}(s_t^{1:N})) \mid  \pi_t, s_t^{1:N}, \gamma_t^{_*1:N}] -\mb{E}[ R_t(X_t, \hat{\gamma}_t^{_01:N}(\hat{s}_t^{1:N})) \mid \hat{\pi}_t, \hat{s}_t^{1:N}, \hat{\gamma}_t^{_01:N}] \big|\leq \epsilon,
	\end{aligned} 
	\end{equation}
	\fi  
	where $\gamma_{t}^{1:N}(s_t^{1:N})$ is a shorthand for $\{\gamma_{t}^{1}(s_t^1),\dots, \gamma_{t}^{N}(s_t^N)\}$. Similarly, by combining the definition of $\hat{\gamma}_{t}^{_01:N}$ with Remark~\ref{remark}, we also have:
	\ifx\singlecolumn\undefined 
	\begin{equation}\label{eqn:belief}
	\begin{aligned} 
	\big| \mb{E}[  \hat{V}(\hat{\Pi}_{t+1}) &\mid  \pi_t, s_t^{1:N}, \gamma_t^{_*1:N}] -\\
	&\mb{E}[  \hat{V}(\hat{\Pi}_{t+1}) \mid \hat{\pi}_t, \hat{s}_t^{1:N}, \hat{\gamma}_t^{_01:N}] \big|\leq L_V \delta.
	\end{aligned} 
	\end{equation}
	\else 
	\begin{equation}\label{eqn:belief}
	\begin{aligned} 
	\big| \mb{E}[  \hat{V}(\hat{\Pi}_{t+1}) \mid  \pi_t, s_t^{1:N}, \gamma_t^{_*1:N}] -\mb{E}[  \hat{V}(\hat{\Pi}_{t+1}) \mid \hat{\pi}_t, \hat{s}_t^{1:N}, \hat{\gamma}_t^{_01:N}] \big|\leq L_V \delta.
	\end{aligned} 
	\end{equation}
	\fi  
	To see this, notice that for any Borel subset $B$ of $\Delta(\mc{X}_{t+1}\times \hat{\mc{S}}_{t+1})$, we have:
	\ifx\singlecolumn\undefined 
	\small
	\[
	\begin{aligned} 
	&\mu_t(B;\gamma_t^{_*1:N}) \triangleq\ \mb{P}\left( \hat{\Pi}_{t+1} \in B \mid \pi_t,s_t^{1:N}, \gamma_t^{_*1:N} \right)\\
	=&\hspace{-1mm}\sum_{a_t^{1:N}\in \mc{A}_t^{1:N}}\hspace{-3mm}\mb{P}\left( \hat{\Pi}_{t+1} \in B \mid \pi_t,s_t^{1:N}, a_t^{1:N} \right) \pp \left(a_t^{1:N}\mid s_t^{1:N},\gamma_t^{_*1:N}\right)\\
	=&\hspace{-1mm}\sum_{a_t^{1:N}\in \mc{A}_t^{1:N}}\hspace{-3mm} \mu_t(B)\pp \left(a_t^{1:N}\mid s_t^{1:N},\gamma_t^{_*1:N}\right),
	\end{aligned} 
	\]
	\normalsize
	\else 
	\[
	\begin{aligned} 
	&\mu_t(B;\gamma_t^{_*1:N}) \\
	\triangleq\ &\mb{P}\left( \hat{\Pi}_{t+1} \in B \mid \pi_t,s_t^{1:N}, \gamma_t^{_*1:N} \right)\\
	=&\hspace{-1mm}\sum_{a_t^{1:N}\in \mc{A}_t^{1:N}}\hspace{-3mm}\mb{P}\left( \hat{\Pi}_{t+1} \in B \mid \pi_t,s_t^{1:N}, a_t^{1:N} \right) \pp \left(a_t^{1:N}\mid s_t^{1:N},\gamma_t^{_*1:N}\right)\\
	=&\hspace{-1mm}\sum_{a_t^{1:N}\in \mc{A}_t^{1:N}}\hspace{-3mm} \mu_t(B)\pp \left(a_t^{1:N}\mid s_t^{1:N},\gamma_t^{_*1:N}\right),
	\end{aligned} 
	\]
	\fi  
	and also note that by Kantorovich-Rubinstein duality,
	\ifx\singlecolumn\undefined 
	\small
	\begin{align}
	&\mc{K}\left(\mu_t(B;\gamma_{t}^{_*1:N}),\nu_t(B; \hat{\gamma}_t^{_01:N})\right)\nonumber\\
	=&\sup_{\|f\|_{\text{Lip}}\leq 1} \left|\int fd\mu_t(\gamma_t^{_*1:N}) - \int f d \nu_t(\hat{\gamma}_t^{_01:N})\right|\nonumber\\
	=&\sup_{\|f\|_{\text{Lip}}\leq 1} \Big|\sum_{a_t^{1:N}\in \mc{A}_t^{1:N}} \pp \left(a_t^{1:N}\mid s_t^{1:N},\gamma_t^{_*1:N}\right) \int fd\mu_t(a_t^{1:N})\nonumber\\
	&\quad -\sum_{a_t^{1:N}\in \mc{A}_t^{1:N}}   \pp \left(a_t^{1:N}\mid \hat{s}_t^{1:N},\hat{\gamma}_t^{_01:N}\right) \int f d \nu_t(a_t^{1:N})\Big|\label{eqn:e1}\\
	\leq&\sup_{\|f\|_{\text{Lip}}\leq 1} \sum_{a_t^{1:N}\in \mc{A}_t^{1:N}} \pp \left(a_t^{1:N}\mid s_t^{1:N},\gamma_t^{_*1:N}\right)
	\nonumber\\
	&\quad \times \bigg|\int fd\mu_t(a_t^{1:N}) - \int f d \nu_t(a_t^{1:N})\bigg|.\label{eqn:e2}
	\end{align}
	\normalsize
	\else 
	\begin{align}
	&\mc{K}\left(\mu_t(B;\gamma_{t}^{_*1:N}),\nu_t(B; \hat{\gamma}_t^{_01:N})\right)\nonumber\\
	=&\sup_{\|f\|_{\text{Lip}}\leq 1} \left|\int fd\mu_t(\gamma_t^{_*1:N}) - \int f d \nu_t(\hat{\gamma}_t^{_01:N})\right|\nonumber\\
	=&\sup_{\|f\|_{\text{Lip}}\leq 1} \Big|\sum_{a_t^{1:N}\in \mc{A}_t^{1:N}} \pp \left(a_t^{1:N}\mid s_t^{1:N},\gamma_t^{_*1:N}\right) \int fd\mu_t(a_t^{1:N})\nonumber\\
	&\qquad\qquad\qquad\qquad  -\sum_{a_t^{1:N}\in \mc{A}_t^{1:N}}   \pp \left(a_t^{1:N}\mid \hat{s}_t^{1:N},\hat{\gamma}_t^{_01:N}\right) \int f d \nu_t(a_t^{1:N})\Big|\label{eqn:e1}\\
	\leq&\sup_{\|f\|_{\text{Lip}}\leq 1} \sum_{a_t^{1:N}\in \mc{A}_t^{1:N}} \pp \left(a_t^{1:N}\mid s_t^{1:N},\gamma_t^{_*1:N}\right) \bigg|\int fd\mu_t(a_t^{1:N}) - \int f d \nu_t(a_t^{1:N})\bigg|.\label{eqn:e2}
	\end{align}
	\fi  
	Equality~\eqref{eqn:e1} holds because the probability measures are finite and the coefficients are non-negative. Inequality~\eqref{eqn:e2} follows from the triangle inequality and the fact that \allowbreak $\pp \left(a_t^{1:N}\mid \hat{s}_t^{1:N},\hat{\gamma}_t^{_01:N}\right) = \pp \left(a_t^{1:N}\mid s_t^{1:N},\gamma_t^{_*1:N}\right),\forall a_t^{1:N}\in \mc{A}_t^{1:N}$. Since the supremum of summation is no larger than the summation of suprema:
	\ifx\singlecolumn\undefined 
	\small
	\begin{align} 
	&\mc{K}\left(\mu_t(B;\gamma_{t}^{_*1:N}),\nu_t(B; \hat{\gamma}_t^{_01:N})\right)\nonumber\\
	\leq& \!\!\!\sum_{a_t^{1:N}\in \mc{A}_t^{1:N}} \!\!\!\pp \left(a_t^{1:N}\mid s_t^{1:N},\gamma_t^{_*1:N}\right)
	\nonumber\\
	&\quad\times \sup_{\|f\|_{\text{Lip}}\leq 1} \left|\int fd\mu_t(a_t^{1:N}) - \int f d \nu_t(a_t^{1:N})\right|\nonumber\\
	\leq &\!\!\!\!\!\sum_{a_t^{1:N}\in \mc{A}_t^{1:N}} \!\!\!\!\!\pp \left(a_t^{1:N}\mid s_t^{1:N},\gamma_t^{_*1:N}\right) \mc{K}\left(\mu_t(a_t^{1:N}), \nu_t^{1:N}(a_t^{1:N})\right)\label{eqn:e3}\\
	\leq& \,\,\delta.\label{eqn:e4}
	\end{align}
	\normalsize
	\else 
	\begin{align} 
	&\mc{K}\left(\mu_t(B;\gamma_{t}^{_*1:N}),\nu_t(B; \hat{\gamma}_t^{_01:N})\right)\nonumber\\
	\leq& \sum_{a_t^{1:N}\in \mc{A}_t^{1:N}} \pp \left(a_t^{1:N}\mid s_t^{1:N},\gamma_t^{_*1:N}\right)\sup_{\|f\|_{\text{Lip}}\leq 1} \left|\int fd\mu_t(a_t^{1:N}) - \int f d \nu_t(a_t^{1:N})\right|\nonumber\\
	\leq &\sum_{a_t^{1:N}\in \mc{A}_t^{1:N}} \pp \left(a_t^{1:N}\mid s_t^{1:N},\gamma_t^{_*1:N}\right) \mc{K}\left(\mu_t(a_t^{1:N}), \nu_t^{1:N}(a_t^{1:N})\right)\label{eqn:e3}\\
	\leq& \delta.\label{eqn:e4}
	\end{align}
	\fi  
	Inequality~\eqref{eqn:e3} holds because Kantorovich-Rubinstein duality implies that $\mc{K}(\mu,\nu)$ is the upper bound of any function with Lipschitz constant no larger than $1$, and hence holds for the specific function $f$.  Finally,~\eqref{eqn:e4} is due to Definition~\ref{def:embedding}(b) and the fact that $\sum_{a_t^{1:N}\in \mc{A}_t^{1:N}} \pp \left(a_t^{1:N}\mid s_t^{1:N},\gamma_t^{_*1:N}\right) = 1$. 
	
	Using this oracle $\hat{\gamma}_{t}^{_01:N}$, for any realization of sufficient private information $s_t^{1:N}$ and common information $c_t$ at time $t$, let $\pi_t = \eta(c_t)$ and $\hat{\pi}_t =  \hat{\eta}(c_t)$; then, we have:
	\ifx\singlecolumn\undefined 
	\small 
	\begin{align}
	&Q_t(\pi_t, \gamma_t^{_*1:N})\\ 
	=& \mathbb{E}  \Big[R_{t}(X_{t}, \gamma_t^{_*1: N}(s_{t}^{1: N}))\nonumber\\
	&\qquad+ V_{t+1}(\psi_{t+1}^{\gamma^*_{t}}(\pi_{t}, Z_{t+1}))\Big| \pi_t, s_t^{1:N}, \gamma_t^{_*1:N}\Big]\label{eqn:t1}\\
	=& \mathbb{E}  \Big[R_{t}(X_{t}, \gamma_t^{_*1: N}(s_{t}^{1: N}))+V_{t+1}(\Pi_{t+1})\mid \pi_t, s_t^{1:N}, \gamma_t^{_*1:N}\Big]\label{eqn:t2}\\
	\leq &\mathbb{E}  \Big[R_{t}(X_{t}, \gamma_t^{_*1: N}(s_{t}^{1: N}))+\hat{V}_{t+1}(\hat{\Pi}_{t+1})\mid \pi_t, s_t^{1:N}, \gamma_t^{_*1:N}\Big] \nonumber\\
	&\qquad + (T-t)(\epsilon+L_V\delta).\label{eqn:t3}
	\end{align}
	\normalsize
	\else 
	\begin{align}
	&Q_t(\pi_t, \gamma_t^{_*1:N})\nonumber\\ 
	=& \mathbb{E}  \Big[R_{t}(X_{t}, \gamma_t^{_*1: N}(s_{t}^{1: N}))+ V_{t+1}(\psi_{t+1}^{\gamma^*_{t}}(\pi_{t}, Z_{t+1}))\Big| \pi_t, s_t^{1:N}, \gamma_t^{_*1:N}\Big]\label{eqn:t1}\\
	=& \mathbb{E}  \Big[R_{t}(X_{t}, \gamma_t^{_*1: N}(s_{t}^{1: N}))+V_{t+1}(\Pi_{t+1})\mid \pi_t, s_t^{1:N}, \gamma_t^{_*1:N}\Big]\label{eqn:t2}\\
	\leq &\mathbb{E}  \Big[R_{t}(X_{t}, \gamma_t^{_*1: N}(s_{t}^{1: N}))+\hat{V}_{t+1}(\hat{\Pi}_{t+1})\mid \pi_t, s_t^{1:N}, \gamma_t^{_*1:N}\Big] + (T-t)(\epsilon+L_V\delta).\label{eqn:t3}
	\end{align}
	\fi  
	Equality~\eqref{eqn:t1} is by the definition of $Q_t$ in the dynamic programming. Equality~\eqref{eqn:t2} is by the definition of $\psi$. Inequality~\eqref{eqn:t3} comes from our induction hypothesis. Using the results from Equations~\eqref{eqn:reward} and~\eqref{eqn:belief}, we then have:
	\ifx\singlecolumn\undefined 
	\small  
	\begin{align} 
	&Q_t(\pi_t, \gamma_t^{_*1:N})\\ 
	\leq &\left(\mathbb{E}  \Big[R_{t}(X_{t}, \hat{\gamma}_t^{_01:N}(\hat{s}_{t}^{1: N}))\mid \hat{\pi}_t, \hat{s}_t^{1:N}, \hat{\gamma}_t^{_01:N}\Big]+\epsilon \right)\nonumber\\
	&\qquad +\left(\mathbb{E}  \Big[\hat{V}_{t+1}(\hat{\Pi}_{t+1})\mid \hat{\pi}_t, \hat{s}_t^{1:N}, \hat{\gamma}_t^{_01:N}\Big] + L_V\delta\right)\nonumber\\
	&\qquad + (T-t)(\epsilon+L_V\delta)\label{eqn:t4}\\
	=& \hat{Q}_t(\hat{\pi}_t, \hat{\gamma}_t^{_01:N}) + (T-t + 1)(\epsilon+L_V\delta)\label{eqn:t5}\\
	\leq &\hat{Q}_t(\hat{\pi}_t, \hat{\gamma}_t^{_*1:N}) + (T-t + 1)(\epsilon+L_V\delta).\label{eqn:t6}
	\end{align}
	\normalsize
	\else 
	\begin{align} 
	&Q_t(\pi_t, \gamma_t^{_*1:N})\nonumber\\ 
	\leq &\left(\mathbb{E}  \Big[R_{t}(X_{t}, \hat{\gamma}_t^{_01:N}(\hat{s}_{t}^{1: N}))\mid \hat{\pi}_t, \hat{s}_t^{1:N}, \hat{\gamma}_t^{_01:N}\Big]+\epsilon \right)\nonumber\\
	&\qquad +\left(\mathbb{E}  \Big[\hat{V}_{t+1}(\hat{\Pi}_{t+1})\mid \hat{\pi}_t, \hat{s}_t^{1:N}, \hat{\gamma}_t^{_01:N}\Big] + L_V\delta\right) + (T-t)(\epsilon+L_V\delta)\label{eqn:t4}\\
	=& \hat{Q}_t(\hat{\pi}_t, \hat{\gamma}_t^{_01:N}) + (T-t + 1)(\epsilon+L_V\delta)\label{eqn:t5}\\
	\leq &\hat{Q}_t(\hat{\pi}_t, \hat{\gamma}_t^{_*1:N}) + (T-t + 1)(\epsilon+L_V\delta).\label{eqn:t6}
	\end{align}
	\fi  
	Equality~\eqref{eqn:t5} follows from the definition of $\hat{Q}_t(\hat{\pi}_t, \hat{\gamma}_t^{_01:N})$. Inequality~\eqref{eqn:t6} holds because $\hat{\gamma}_t^{_*1:N}$ is optimal to the embedded dynamic program, and hence its $\hat{Q}_t$ value is no smaller than that of $\hat{\gamma}_t^{_01:N}$.
\end{proof}
Our result suggests that the error of carrying out dynamic programming using only the embedded information is upper bounded linearly in time from dynamic programming with the full information. To obtain a small value error upper bound, embedding schemes with small compression errors ($\epsilon$ and $\delta$) should be designed. 

\subsection{Learning an Information State Embedding}\label{sec:learning}
In this subsection, we introduce an empirical instance of the information state embedding based on Recurrent Neural Networks (RNNs), and demonstrate how to learn such an embedding from data. A theoretical upper bound for the $(\epsilon, \delta)$ values of this embedding is generally unclear, as it requires quantification of the expressiveness of the RNNs, which is beyond the scope of this paper. Instead, the embedding is introduced to demonstrate the feasibility of the embed-then-learn framework. We evaluate its performance empirically in the next section. 

The recurrent neural network embedding (RNN-E) uses an RNN to compress the history. In our implementation, each agent uses an LSTM network~\cite{hochreiter1997long} (a variant of RNN) that maps its local history to a fixed-size vector at each time step. We treat the fixed-size hidden state of the LSTM network as our information state embedding. Recursive update is inherent in the structure of LSTM: $(\hat{s}_{t+1}^i,\hat{\pi}_{t+1}) = \text{lstm}_i\left(\hat{s}_t^i, \hat{\pi}_t, y_{t+1}^i, a_t^i, z_{t+1}, \text{Cell}_t^i ; W_i\right)$, where $\text{Cell}_t^i$ is the cell state of the LSTM network that keeps a selective memory of history, and $W_i$ is the network parameters to be learned from data. 

Once we have an embedding, we use deep Q-networks (DQNs)~\cite{mnih2015human} to learn a policy. Following the embed-then-learn procedure, as illustrated in Fig.~\ref{fig:pipeline}, we feed the embedding into a DQN to get the Q-value for each candidate action. In the single agent setting, a similar network structure, named DRQN~\cite{hausknecht2015deep},  concatenates LSTM and DQN, and adopts an end-to-end structure where LSTM directly outputs the Q-values. In contrast, we extract an embedding first so that we can theoretically bound the value function given an upper bound on the embedding error.

Similar to the problem faced by other MARL algorithms, the environment is non-stationary from each agent's perspective since agents learn and update policies concurrently. To address this issue, common training schemes in the literature include centralized training and execution, concurrent learning, and parameter sharing~\cite{gupta2017cooperative}. In our simulations, we adopt parameter sharing as it has demonstrated better performance in~\cite{gupta2017cooperative}. In parameter sharing, homogeneous agents share the same network parameter values, which leads to more efficient training and partly addresses the non-stationarity issue in concurrent learning. Heterogeneous policies are still possible because agents feed unique agent IDs and different local observations into the network. 
However, as a standard training scheme in the literature, parameter sharing does slightly break the assumption of decentralization, as it requires either centralized learning (but still fully decentralized execution), or periodic gradients sharing among agents (which is still a weaker assumption than real-time sharing of local observations).

\section{Numerical Results}\label{sec:numerical}
In this section, we evaluate our embedding scheme on several benchmark problems in the Dec-POMDP literature~\cite{masplan}: Grid3x3corners~\cite{amato2009incremental}, Dectiger~\cite{nair2003taming}, and Boxpushing~\cite{seuken2007improved}. 

We first introduce two heuristics that can also serve to compress the history, but generally do not satisfy our definition of an embedding (more specifically, they are not injective). These two compression heuristics are followed by a DQN to learn a policy in exactly the same way as RNN-E. We will use these two heuristics as baselines to evaluate the performance of our embedding. 

\vspace{3pt}
\noindent\textbf{Finite memory compression.} The first heuristic compression, finite memory compression (FMC), simply maintains a fixed memory, or window, of the local history as an approximate information state.\footnote{Although finite-memory decision making has been well studied in the single agent setting~\cite{white1994finite}, it is still generally open how the truncation of history influences the performance in the multi-agent setting.} Specifically, each agent maintains a one-hot encoded vector of a fixed window of its most recent actions and observations, and its decision only depends on this fixed memory. This compression can be updated recursively via $\hat{s}_{t+1}^i = \left(\hat{s}_{t}^i\backslash \{y_{t-M+1}^i, a_{t-M}^i, z_{t-M+1}\}\right) \cup \{y_{t+1}^i, a_t^i, z_{t+1}\}$, where $M$ is the length of the fixed window. This compression can be regarded as a simplification of~\cite{banerjee2012sample}, where the authors define each complete history sequence to be an information state and perform $Q$-learning~\cite{watkins1992q} on such an information state space. Their method does not scale well to longer horizons due to the explosion of the new state space; in contrast, FMC is more scalable as its size is fixed, but it comes at a price of losing long-term memory.

\vspace{3pt}
\noindent
\textbf{Principal component analysis compression.} The second heuristic compression, principal component analysis compression (PCAC), uses PCA~\cite{pearson1901liii} to reduce the local history to a fixed-size feature vector. PCA is a simple and well-established algorithm for dimensionality reduction. Given a specified dimensionality, PCA keeps the largest variance during compression. Note that this goal differs from our intention of maintaining the predictive capability for decision-making purposes. PCAC cannot be defined a priori and needs to be first trained on a data set. In our implementation, we generate this training set by uniformly randomizing history sequences of the same length. We also note that PCA can be implemented recursively~\cite{li2000recursive} to handle sequential data.

We compare the performance of our RNN-based embedding with the above two compression heuristics, the state-of-the-art planning solution FB-HSVI~\cite{dibangoye2016optimally}, which requires a complete model of the environment, and a learning solution oSARSA~\cite{dibangoye2018learning}. We refer to the performances of FB-HSVI and oSARSA, as reported by their authors. The authors limited the running time of their algorithms to $10^5$ episodes and $5$ hours, but these stopping criteria are not the binding constraints for our solutions, as our algorithm takes significantly less time and fewer episodes to converge.

For RNN-E, we use a one-layer LSTM network with a hidden layer of size $10$ as the embedding network. The inputs to the LSTM are one-hot encoded actions and observations, together with the embedding from the last step. Our DQN is a two-layer fully connected network, where the input is the embedding/compression. The DQN hidden layer has $10$ neurons, and the output size is equal to the size of action space. All activations are Rectified Linear Unit (ReLU)~\cite{glorot2011deep} functions. 

We adopt the $\epsilon$-greedy approach for policy exploration, with $\epsilon$ decreasing linearly from $0.8$ to $0$ over the total $40000$ episodes by default. We use a buffer of size $4000$ for experience replay, and for DQN error estimation we draw a batch of $400$ samples from the buffer. We use Mean Squared Error loss and the Adam optimizer, with a learning rate of $10^{-2}$ by default for both the embedding network and DQN. We perform back-propagation after each episode, and update the target network of DQN every $100$ episodes. The size of the compressed information state for FMC and PCAC is set to $M=4$. For more efficient training, we only train our networks with a horizon length of $10$ and then test on different horizons, which surprisingly achieves comparable performance as training and testing on each possible horizon length separately. We average the performance over $2000$ testing episodes in each run, and all results are averaged over $10$ runs. The performances of the five algorithms over different lengths of horizon $T$ are shown in TABLE~\ref{tbl:exp}. 

\ifx\singlecolumn\undefined 
\begin{table}[!htb]
	\centering
		\begin{tabular}{@{}cccccc@{}}
			\toprule
			& \multicolumn{3}{c}{Parameter Sharing} & \multicolumn{2}{c}{Centralized Solutions}\\
			\cmidrule(lr{0.5em}){2-4} \cmidrule(l{0.5em}){5-6}
			$T$     & RNN-E & FMC & PCAC & oSARSA & FB-HSVI \\ \midrule
			\multicolumn{6}{c}{Grid3x3corners}            \\ \midrule
			6     &   0.86     &   0.83  &0.26  &     1.49       &    1.49     \\
			7     &   1.41     &   1.30  &0.51  &      2.19      &    2.19     \\
			8     &   1.94     &   1.93  &0.72  &     2.95       &    2.96     \\
			9     &   2.69     &   2.53  &1.01  &     3.80       &    3.80     \\
			10    &   3.47     &   3.25   &1.30 &     4.69       &    4.68   
			\\ \midrule
			\multicolumn{6}{c}{Dectiger}                  \\ \midrule
			3     &   4.58     &  4.89  & 0.06  &       5.19     &    5.19     \\
			4     &   2.97     &    3.78  &1.00  &      4.80      &    4.80     \\
			5     &   1.46     &   2.02   &0.65 &      6.99      &    7.02     \\
			6     &   2.50     &   2.95    &0.71&       2.34     &    10.38     \\
			7     &   0.85     &   1.89   &0.53 &      2.25      &    9.99     \\
			   \midrule
			\multicolumn{6}{c}{Boxpushing}                \\ \midrule
			3     &   12.63     &   64.92  &17.81  &     65.27       &    66.08     \\
			4     &   65.06     &   76.83  & 17.76  &    98.16        &    98.59     \\
			5     &   81.51     &   94.22  &34.28  &    107.64        &   107.72      \\
			6     &   91.00     &   97.03  &34.65  &    120.26        &   120.67      \\
			7     &   106.76     &   143.53 & 34.23   &    155.21        &   156.42      \\\bottomrule
		\end{tabular}
	\caption{Performance on classic benchmark problems}
	\label{tbl:exp}
\end{table}
\else 
\begin{table}[!htb]
	\centering
	\resizebox{.6\columnwidth}{!}{%
		\begin{tabular}{@{}cccccc@{}}
			\toprule
			& \multicolumn{3}{c}{Parameter Sharing} & \multicolumn{2}{c}{Centralized Learning}\\
			\cmidrule(lr{0.5em}){2-4} \cmidrule(l{0.5em}){5-6}
			$T$     & RNN-E & FMC & PCAC & oSARSA & FB-HSVI \\ \midrule
			\multicolumn{6}{c}{Grid3x3corners}            \\ \midrule
			6     &   0.86     &   0.83  &0.26  &     1.49       &    1.49     \\
			7     &   1.41     &   1.30  &0.51  &      2.19      &    2.19     \\
			8     &   1.94     &   1.93  &0.72  &     2.95       &    2.96     \\
			9     &   2.69     &   2.53  &1.01  &     3.80       &    3.80     \\
			10    &   3.47     &   3.25   &1.30 &     4.69       &    4.68  
			\\
			\midrule
			\multicolumn{6}{c}{Dectiger}                  \\ \midrule
			3     &   4.58     &  4.89  & 0.06  &       5.19     &    5.19     \\
			4     &   2.97     &    3.78  &1.00  &      4.80      &    4.80     \\
			5     &   1.46     &   2.02   &0.65 &      6.99      &    7.02     \\
			6     &   2.50     &   2.95    &0.71&       2.34     &    10.38     \\
			7     &   0.85     &   1.89   &0.53 &      2.25      &    9.99     
			   \\ \midrule
			\multicolumn{6}{c}{Boxpushing}                \\ \midrule
			3     &   12.63     &   64.92  &17.81  &     65.27       &    66.08     \\
			4     &   65.06     &   76.83  & 17.76  &    98.16        &    98.59     \\
			5     &   81.51     &   94.22  &34.28  &    107.64        &   107.72      \\
			6     &   91.00     &   97.03  &34.65  &    120.26        &   120.67      \\
			7     &   106.76     &   143.53 & 34.23   &    155.21        &   156.42      \\\bottomrule
		\end{tabular}
	}
	\caption{Performance on classic benchmark problems}
	\label{tbl:exp}
\end{table}
\fi  

\ifx\singlecolumn\undefined 
\begin{figure}[!htb]
	\centering\includegraphics[width=.18\textheight]{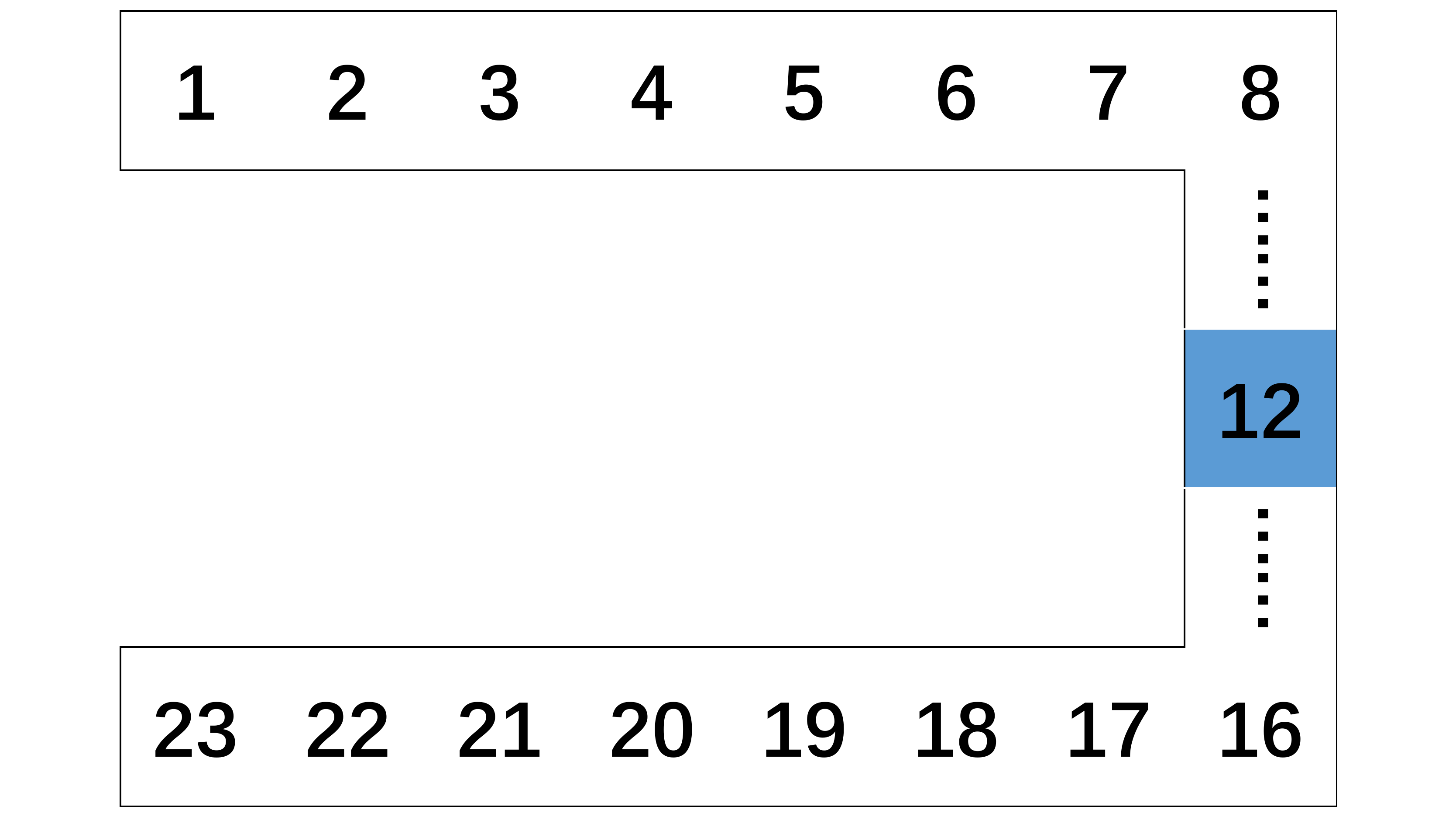}
	\caption{A Dec-POMDP example requiring long-term memory}\label{fig:example}
\end{figure}
\else 
\begin{figure}[!htb]
	\centering\includegraphics[width=.3\textheight]{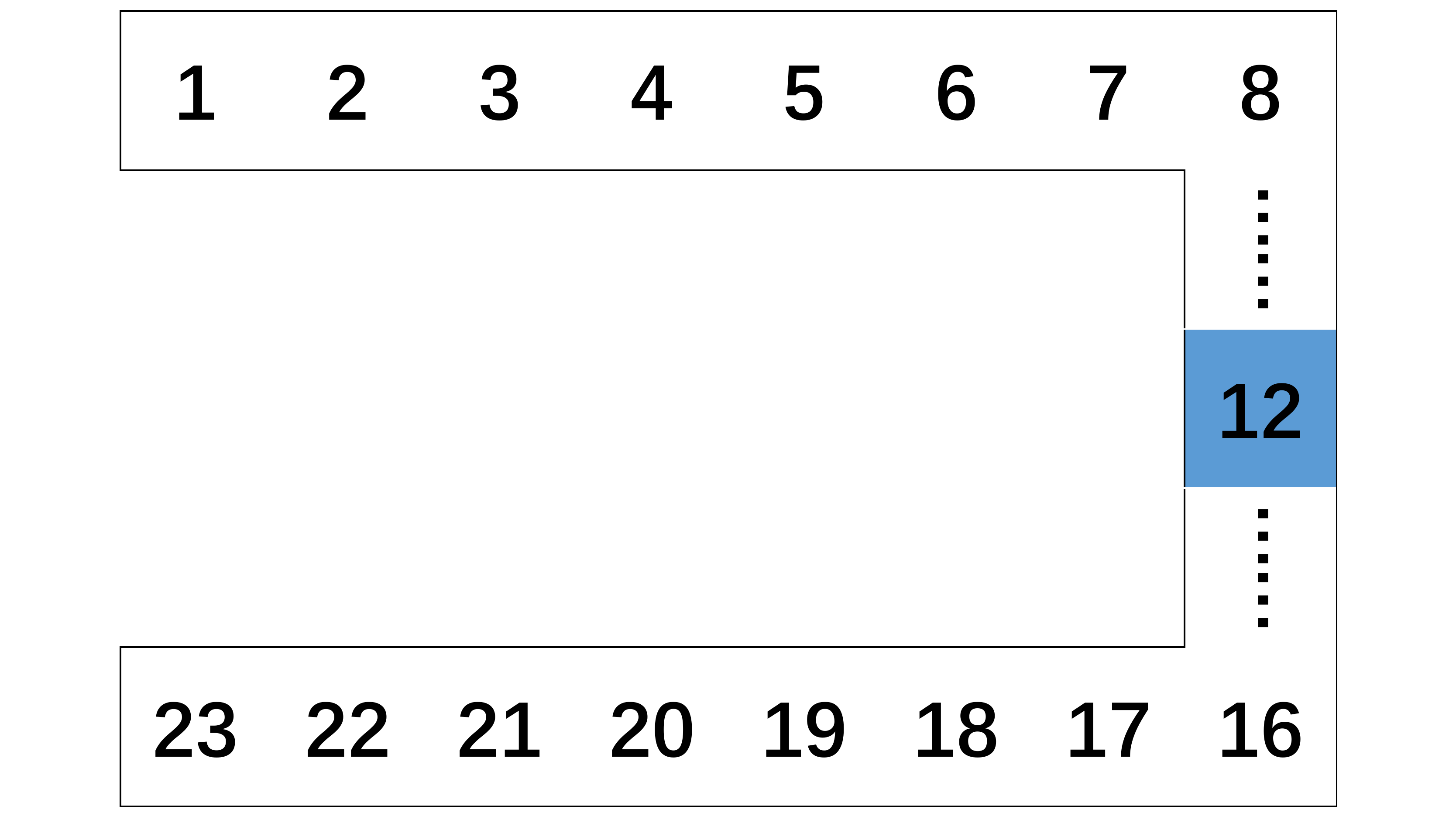}
	\caption{A Dec-POMDP example requiring long-term memory}\label{fig:example}
\end{figure}
\fi

We can see that RNN-E and FMC achieve high rewards over different horizons. FMC performs better on shorter horizon problems Dectiger and Boxpushing, but RNN-E outperforms FMC on Grid3x3corners, where horizons are longer and long-term memories are necessary. Although FMC achieves good performance in the three examples, we note that it has a very limited scope of application, because it is easy to construct examples where short-term memories are not sufficient for decision making. For example, consider a two-agent Dec-POMDP problem as illustrated in Fig.~\ref{fig:example}, which can be regarded as a modification of Grid3x3corners. Agent 1 starts from state 1 in the maze, and Agent 2 starts from state 23. The goal for the two agents is to meet at the destination state 12 as soon as possible (i.e., they receive a time-discounted unit reward when both of them are in state 12, and no reward otherwise). Candidate actions include moving one step in any of the four directions. They always receive the same observation no matter what states they are in and what actions they take. If an agent runs into a wall, it stays where it is. For each agent, we can see that it is sufficient to only count how many times it has been going right, and the optimal strategy is to switch from going right to going down / up when the count reaches $7$. Now suppose Agent 1 only has a finite memory of length 4. Then this agent performs poorly because it cannot distinguish states 5, 6, 7, and 8. If it decides to deterministically go right, it will get stuck in state 8 forever. If it has some probability of going down, then it wastes time in states 5, 6, and 7. Therefore, finite memory agents obtain low rewards in this example. On the other hand, RNN-E is able to summarize the entire history rather than only keeping a short-term memory, but it comes at a price that RNNs are generally difficult to train mostly due to the vanishing and exploding gradient problems.

PCAC does not perform well on the three tasks. We believe the reason is that PCA is designed to keep the largest possible variance of data, which is generally not the same as the most predictive information of the history as required by Definition~\ref{def:embedding}. The oSARSA algorithm performs comparably well as the planning solution FB-HSVI, and generally outperforms our solutions. This is because oSARSA relies on centralized learning, which is a much stronger assumption than the parameter sharing assumption that our solutions rely on. The centralized scheme of oSARSA also incurs heavy computation, as it requires to solve a mixed-integer linear program at each step.

\section{Concluding Remarks and Future Directions}\label{sec:conclusions}
In this paper, we have introduced the concept of information state embedding for partially observable cooperative MARL. We have theoretically analyzed how the compression error of the embedding influences the value functions. We have also proposed an instance of embedding based on RNNs, and empirically evaluated its performance on partially observable MARL benchmarks. An interesting future direction would be to theoretically analyze the compression errors of the embedding/compression strategies we have used, which helps close the loop of our theoretical analysis.  It would also be interesting to design other empirical embeddings that explicitly reduce this compression error.

	\balance 
	\bibliographystyle{IEEEtran}
	\bibliography{ref}

\end{document}